\documentclass[conference]{IEEEtran}
\IEEEoverridecommandlockouts
\usepackage{cite}
\usepackage{amsmath,amssymb,amsfonts,amsthm}
\usepackage{stmaryrd} 
\usepackage{graphicx}
\usepackage{textcomp}
\usepackage{xcolor}
\def\BibTeX{{\rm B\kern-.05em{\sc i\kern-.025em b}\kern-.08em
    T\kern-.1667em\lower.7ex\hbox{E}\kern-.125emX}}

\usepackage[linesnumbered,ruled,vlined]{algorithm2e}

\usepackage{bm}
\usepackage{todonotes}
\usepackage{flushend} 

\DeclareMathOperator{\tr}{tr}
\DeclareMathOperator{\diag}{diag}

\DeclareMathOperator{\rank}{rank}

\newtheorem{theorem}{Theorem}
\newtheorem{lemma}[theorem]{Lemma}

\graphicspath{{./fig/}}
\usepackage[labelformat=simple]{subcaption} 
\newcounter{mycounter}

\begin{document}

\title{Including Node Textual Metadata in Laplacian-constrained Gaussian Graphical Models
	\thanks{This work was supported by the MASSILIA project (ANR21-CE23-0038-01) of the French National Research Agency (ANR).}
}

\author{\IEEEauthorblockN{Jianhua Wang$^{*}$$^{\dagger}$, Killian Cressant$^{\dagger}$, Pedro Braconnot Velloso$^{\dagger}$, Arnaud Breloy$^{\dagger}$}
	\IEEEauthorblockA{\textit{$^{*}$ Université Paris-Nanterre, LEME, IUT Ville d’Avray, Ville d’Avray, France} \\
		\textit{$^{\dagger}$ Conservatoire National des Arts et Métiers, CNAM, Paris, France}
	}
}

\maketitle

\thispagestyle{plain}
\pagestyle{plain}

\begin{abstract}
	This paper addresses graph learning in Gaussian Graphical Models (GGMs). In this context, data matrices often come with auxiliary metadata (e.g., textual descriptions associated with each node) that is usually ignored in traditional graph estimation processes. To fill this gap, we propose a graph learning approach based on Laplacian-constrained GGMs that jointly leverages the node signals and such metadata. The resulting formulation yields an optimization problem, for which we develop an efficient majorization-minimization (MM) algorithm with closed-form updates at each iteration. Experimental results on a real-world financial dataset demonstrate that the proposed method significantly improves graph clustering performance compared to state-of-the-art approaches that use either signals or metadata alone, thus illustrating the interest of fusing both sources of information.
\end{abstract}


\section{Introduction}

Graphs are fundamental mathematical structures in both theoretical and applied sciences, providing a natural framework to model entities (nodes) and their relationships (edges).
This representation enables the development of a wide range of methodologies for classical tasks such as signal filtering, anomaly detection, clustering, and classification.
Representative examples include graph clustering methods \cite{hollocou2019modularity, wang2023overview}, graph signal processing \cite{ortega2018graph, isufi2024graph, leus2023graph}, and graph neural networks \cite{wu2020comprehensive}.

Most of the aforementioned tools, however, are built upon the assumption that the underlying graph topology is known.
In many practical scenarios, this assumption does not hold, and the graph structure must instead be inferred from data, giving rise to the problem of graph learning.
A prominent line of research in statistical learning relates the graph topology to the conditional dependency structure among variables associated with the nodes.
Within this framework, Gaussian graphical models (GGM), also referred to as Gaussian Markov random fields (GMRF), assume that the data is sampled from a multivariate Gaussian distribution and allow us to estimate the graph from the support of the precision matrix (the inverse of the covariance matrix) \cite{lauritzen1996graphical}.
Furthermore, Laplacian-constrained GGMs impose the estimated precision to be a Laplacian matrix \cite{egilmez2017graph}.
This framework bridges statistics (Gaussian models) to the field of graph signal processing\cite{shuman2013emerging}: the resulting learned graphs are interpreted as the ones who favor smooth signal representations \cite{dong2016learning, kalofolias2016learn}. 
Learned eigenvectors of the precision matrix are also linked to a graph Fourier basis \cite{ kumar2019structured, isufi2024graph}.

In many real-world applications, additional side information describing node attributes is often available.
For example, textual metadata such as variable descriptions can be appended to the data matrix \cite{lake2010discovering, ying2020nonconvex, phi2026leveraging, hippert2023learning}.
Classical graph learning methods usually discard such auxiliary information.
This could result in degraded graph estimates and limit their effectiveness in downstream tasks.
To address this limitation: 
\begin{itemize}
	\item We investigate the problem of learning a Laplacian matrix within the GMRF framework \cite{egilmez2017graph, ying2020nonconvex}, while explicitly incorporating side information in the objective function.
	
	\item We develop an efficient optimization algorithm based on the majorization--minimization (MM) principle \cite{sun2016majorization}, leading to a computationally efficient algorithm with closed-form solutions at each iteration step.
	
	\item We illustrate the interest of our method on the real-world financial dataset.
	Notably, it allows for improving the clustering performance compared to state-of-the-art approaches that use either the graph signals, or the metadata only.
	
\end{itemize}

\section{Background}

We consider an undirected, weighted graph represented by the triplet $\mathcal{G} = (\mathcal{V}, \mathcal{E}, \mathbf{W})$, where $\mathcal{V} = \{1,2,\ldots,p\}$ denotes the set of vertices (nodes), and $\mathcal{E} \subseteq \{\{u,v\} : u,v \in \mathcal{V},\, u \neq v\}$ is the edge set, i.e., a subset of all possible unordered pairs of nodes.
The matrix $\mathbf{W} \in \mathbb{R}_{+}^{p \times p}$ denotes the symmetric weighted adjacency matrix satisfying
$W_{ii} = 0$, $W_{ij} > 0$ if $\{i,j\} \in \mathcal{E}$, and $W_{ij} = 0$ otherwise.
The combinatorial graph Laplacian matrix $\mathbf{L}$ is defined as
$
	\mathbf{L} \triangleq \mathbf{D} - \mathbf{W},
$
where $\mathbf{D} \triangleq \diag(\mathbf{W}\mathbf{1})$ is the degree matrix.

In this work, we restrict our attention to estimate a combinatorial Laplacian graph matrix, where the set of Laplacian matrices associated with connected graphs can be defined as
\begin{multline}
	\mathcal{S}_L =
	\{
	\mathbf{L} \in \mathcal{S}_{+}^{p} \, \big|\,
	L_{ij} = L_{ji} \leq 0,\ \forall\, i \neq j,\\
	\mathbf{L}\mathbf{1} = \mathbf{0},
	\rank(\mathbf{L}) = p - 1
	\}
	\label{eq:laplacian_set}
\end{multline}
The objective of the sparse graph learning under the Laplacian-constrained GMRF is to estimate the precision of a Gaussian model $\mathbf{x}\sim\mathcal{N}(\mathbf{0},\mathbf{\Sigma})$
under the constraint that $\mathbf{\Sigma}^+=\mathbf{L} \in \mathcal{S}_L$.
To define a parameterization that respects this constraint, we will adopt the linear operator $\mathcal{L}(\cdot)$ from {Definition~3.2} of~\cite{ying2020nonconvex}, which maps a vector $\mathbf{w} \in \mathbb{R}^{m}$ ($m = p(p - 1)/2$) to a matrix $\mathcal{L}(\mathbf{w}) \in \mathbb{R}^{p \times p}$. 
Concretely, this operator simply maps the vectorization of the upper triangular element of adjacency matrix $\mathbf{W}$ to the corresponding Laplacian matrix.
The Laplacian set $\mathcal{S}_L$, defined in~\eqref{eq:laplacian_set}, can thus be equivalently expressed as
\begin{equation}
	\mathcal{S}_L
	=
	\left\{
	\mathcal{L}(\mathbf{w}) \in \mathcal{S}_{+}^{p} \,\big|\,
	\mathbf{w}\in (\mathbb{R}^+)^{m}
	\right\},
	\label{eq:laplacian_set_new}
\end{equation}
where the element-wise constraint $\mathbf{w} \geq \mathbf{0}$ enforces the non-negativity of all edge weights. In the following, $\mathcal{L}\mathbf{w}$ will be used in place of $\mathcal{L}(\mathbf{w})$ for notional simplicity.
In conclusion, instead of directly optimizing the Laplacian matrix $\mathbf{L}$, we will optimize the vector $\mathbf{w}$ 

\section{Problem formulation}

The data matrix is denoted by $
\mathbf{X} = [\mathbf{x}_1, \mathbf{x}_2, \ldots, \mathbf{x}_n]
\in \mathbb{R}^{p \times n}$, with $\mathbf{x}_i\in \mathbb{R}^p$ being one  observation of a graph signal on all the $p$ nodes.
The objective is to learn the non-negative graph weight vector $\mathbf{w}$ from the signals $\mathbf{X}$. 
In many practical scenarios, additional side-information associated with the graph nodes is also available. 
We denote the embedding of this side-information for node $i$ by a vector $\mathbf{y}_i \in \mathbb{R}^{d}$, where $d$ may vary depending on the embedding method.

A representative example we will use in this paper's experiments is SP500 stock dataset.
In that case node represent companies, graph signals are their stock market values, and node metadata consists in textual descriptions of each stock (e.g., the company's principal activities extracted from the official website). 
These textual metadata can then be transformed into embedding vectors $\mathbf{y}_i$ using representation learning techniques such as Word2vec \cite{johnson2024detailed}.

\paragraph{Graph learning from smooth signals}
It is commonly assumed that the observed signals are smooth over the underlying graph.
In Laplacian constrained GMRF, the data is assumed to be zero-mean Gaussian, i.e., $\mathbf{x} \sim \mathcal{N}(\mathbf{0}, \boldsymbol{\Sigma})$, in which $ \boldsymbol{\Sigma}^{+}= \mathbf{L}$ is the precision matrix.
Hence the precision matrix is directly identified to the graph Laplacian.
Sparse graph learning under the Laplacian-constrained GGM can be formulated as a penalized maximum likelihood estimation problem. 
Following \cite{ying2020nonconvex}, the objective function to be minimized is given by
\begin{equation}
	\label{prob:signal}
	f_{\mathbf{X}}^{\lambda}(\mathbf{w}) =   \tr(\mathbf{S}\mathcal{L}\mathbf{w}) - \log \det(\mathcal{L}\mathbf{w} + \mathbf{J}) + \sum_{i} h_\lambda(\mathbf{w}_{i})
\end{equation}
where $
\mathbf{S} = \frac{1}{n} \sum_{t=1}^n \mathbf{x}_t \mathbf{x}_t^{\top}
$ denotes the sample covariance matrix and $h_\lambda(\cdot)$ is a 
sparsity promoting penalty, with regularization parameter $\lambda \geq 0$.
As advocated in \cite{ying2020nonconvex}, we adopt the non-convex penalty SCAD \cite{fan2001variable}.
The matrix $\mathbf{J} = \tfrac{1}{p}\mathbf{1}\mathbf{1}^\top$ ensures that the log-det term is well defined~\cite{egilmez2017graph}.

\paragraph{Gaussian kernel graph from side information}
In contrast to the observed signals, the side-information vectors $\mathbf{y}_i$ are generally not smooth over the graph. 
For instance, textual descriptions transformed into embeddings do not necessarily satisfy a smoothness assumption. 
However, a spatial graph is often of relevance for such embeddings.
This can be obtained from the Gaussian kernel similarities, leading to the adjacency matrix:
\begin{equation}
	\label{eq:Gaussian_kernal_simi}
	\mathbf{W}_{i,j} = \exp\left(-\frac{\mathbf{Z}_{i,j}}{\sigma^2} \right),
\end{equation}
with $\mathbf{Z}_{i,j} = \lVert \mathbf{y}_i - \mathbf{y}_j \rVert_2^2$ encoding semantic distances between node features, and where $\sigma^2$ allows for adjusting the chosen width of a spatial neighborhood.

In fact, the Gaussian kernel similarity in \eqref{eq:Gaussian_kernal_simi} can be interpreted as the closed-form solution of a more general graph weight optimization problem proposed in \cite{kalofolias2016learn}. 
The corresponding objective function is written as
\begin{equation}
	\label{prob:label:Y_gaussian_kerkel}
	f_{\mathbf{Y}}^{\sigma}(\mathbf{w})  =  \mathbf{w}^\top \mathbf{z} + \sigma^2 \sum_{i} \mathbf{w}_{i} \left(\log(\mathbf{w}_{i}) -1\right),
\end{equation}
where $\mathbf{w}$ is constrained to have positive entries, and where $\mathbf{z} \in \mathbb{R}^{m}_{+}$ collects the upper-triangular elements of $\mathbf{Z}$.

\paragraph{Learning from smooth signals with side-information}

To interpolate between the two aforementioned approaches, we propose to jointly integrate signal observations and side information within a unified optimization framework.
The resulting objective function of the graph learning problem is formulated as
\begin{equation}
	\label{prob:whole:init}
	f(\mathbf{w}) = \alpha
	f_{\mathbf{X}}^{\lambda}(\mathbf{w})  +  (1-\alpha)f_{\mathbf{Y}}^{\sigma}(\mathbf{w})
\end{equation}
where the hyperparameter $\alpha \in [0,1]$ controls the relative confidence assigned to the signal-driven and side-information-driven terms.
For analytical convenience, we rewrite the objective function as
\begin{equation}
			f(\mathbf{w}) = {}   \alpha {f_1(\mathbf{w})} + (1-\alpha) {f_2(\mathbf{w})} + \alpha {f_3(\mathbf{w})} 
\end{equation}
in which we define
\begin{equation}
	\label{prob:whole:init_trans}
		\begin{split}	
			  {} & f_1(\mathbf{w})  \triangleq   - \log \det(\mathcal{L}\mathbf{w} + \mathbf{J}) + \tr(\mathbf{S}\mathcal{L}\mathbf{w})  \\
			{} & f_2(\mathbf{w})  \triangleq  \mathbf{w}^\top \mathbf{z} + \sigma^2 \sum_{i} \mathbf{w}_{i} \left(\log(\mathbf{w}_{i}) -1\right)  \\
			 {} & f_3(\mathbf{w})  \triangleq  \sum_{i} h_\lambda(\mathbf{w}_{i}).
		\end{split}
\end{equation}

\section{Proposed algorithm}

To solve the problem, we adopt the majorization-minimization (MM) framework~\cite{sun2016majorization}, which iteratively minimizes a sequence of surrogate functions.
Each MM iteration consists of two steps.
In the majorization step, a surrogate function
$f(\mathbf{w}\,|\,\mathbf{w}^{(k)})$ is constructed to locally upper-bound the original objective $F(\mathbf{w})$ at the current iterate $\mathbf{w}^{(k)}$, such that
\begin{equation}
	\label{eq:mm_major}
	f(\mathbf{w}\,|\,\mathbf{w}^{(k)}) \geq F(\mathbf{w}), \quad
	f(\mathbf{w}^{(k)}\,|\,\mathbf{w}^{(k)}) = F(\mathbf{w}^{(k)}).
\end{equation}
In the minimization step, the surrogate function
$f(\mathbf{w}\,|\,\mathbf{w}^{(k)})$ is minimized to produce the next iterate $\mathbf{w}^{(k+1)}$
\begin{equation}
	\label{eq:mm_minim}
	\mathbf{w}^{(k+1)} = \arg \min_{\mathbf{w}} 	f(\mathbf{w}\,|\,\mathbf{w}^{(k)}).
\end{equation}
From~\eqref{eq:mm_major} and~\eqref{eq:mm_minim}, it immediately follows that
\begin{equation}
	F(\mathbf{w}^{(k+1)})
	\leq
	f(\mathbf{w}^{(k+1)}\,|\,\mathbf{w}^{(k)})
	\leq
	f(\mathbf{w}^{(k)} \mid \mathbf{w}^{(k)})
	=
	F(\mathbf{w}^{(k)}),
	\label{eq:mm_descent}
\end{equation}
which implies that the sequence $\{F(\mathbf{w}^{(k)})\}_{k \geq 0}$ generated by the MM algorithm is non-increasing.

In the following, we derive an MM algorithm to minimize the objective in \eqref{prob:whole:init_trans}.
The key step is to construct a tight surrogate function for each component of the objective.
To this end, we first establish the following lemmas.

\begin{lemma}\label{lemma:1}
	The function $f_1(\mathbf{w}) $ in \eqref{prob:whole:init_trans} admits the following upper bound at a given point $\mathbf{w}_0$
	\begin{equation}
		\label{eq:lemma:f1}
		f_1(\mathbf{w}) \le  \tr(\mathbf{R}\diag(\mathbf{w})) +  \tr\left(\mathbf{Q} \diag(\tilde{\mathbf{w}})^{-1} \right) 
	\end{equation}
	with equality for $\mathbf{w} = \mathbf{w}_0$, where
	\begin{equation}
		\label{eq:update_Q}
		\mathbf{Q} = \diag(\tilde{\mathbf{w}}_0) \mathbf{G}^\top (\mathbf{G} \diag(\tilde{\mathbf{w}}_0) \mathbf{G}^\top)^{-1} \mathbf{G} \diag(\tilde{\mathbf{w}}_0),
	\end{equation}
	and $\mathbf{R} = \mathbf{E}^\top \mathbf{S} \mathbf{E}$ with $\tilde{\mathbf{w}}_0 \triangleq [\mathbf{w}_0^\top, 1/p]^\top$, $\tilde{\mathbf{w}} = [\mathbf{w}^\top, 1/p]^\top\in \mathbb{R}^{m+1}$ and $\mathbf{G} = [\mathbf{E}, \mathbf{1}]\in \mathbb{R}^{p \times (m+1)}$. The matrix $\mathbf{E} = [\boldsymbol{\xi}_1, \ldots, \boldsymbol{\xi}_m] \in \mathbb{R}^{p \times m}$ consists of column vectors $\boldsymbol{\xi}_k = \mathbf{e}_{i,k} - \mathbf{e}_{j,k}$ corresponding to the edge $(i, j)$ with $i > j$ and $\mathbf{e}_{i,k} $ is the $i$-th canonical basis vector. The index $k$ maps the edge to the column of $\mathbf{E}$ via 
	$ k = i - j + \frac{j-1}{2}(2p - j)	$. 
\end{lemma}

\begin{proof} 
	The proof of this lemma comes from \cite{javaheri2024joint}, and is reported to introduce some required matrix notations.
	The log-determinant function is concave over the cone of positive definite matrices and therefore satisfies the following first-order upper-bound inequality
	\begin{equation}
		\log \det (\mathbf{X}) \leq \log \det (\mathbf{X}_0) + \tr \!\left( \mathbf{X}_0^{-1} (\mathbf{X} - \mathbf{X}_0) \right),
		\label{eq:logdet_concavity}
	\end{equation}
	for any $\mathbf{X} \succ 0$ and any reference point $\mathbf{X}_0 \succ 0$~\cite{boyd2004convex}.
	Applying this inequality yields a majorizer for the negative log-determinant term 
	\begin{align}
		- \log \det (\mathcal{L}\mathbf{w} + \mathbf{J})
		&= \log \det \left( (\mathcal{L}\mathbf{w} + \mathbf{J})^{-1} \right) \nonumber \\
		&\leq \tr \left( \mathbf{F}_0 (\mathcal{L}\mathbf{w} + \mathbf{J})^{-1} \right) + \text{const.},
		\label{eq:inverse_logdet_majorization}
	\end{align}
	where $\mathbf{F}_0 = \mathcal{L}\mathbf{w}_0 + \mathbf{J}$.

	Using the Laplacian factorization~\cite{kumar2020unified},
	\begin{equation}
		\mathcal{L}\mathbf{w} + \mathbf{J}
		= \mathbf{E} \diag(\mathbf{w}) \mathbf{E}^\top + \mathbf{J}
		= \mathbf{G} \diag(\tilde{\mathbf{w}}) \mathbf{G}^\top,
		\label{eq:laplacian_factorization}
	\end{equation}
	and the cyclic property of the trace, we obtain
	\begin{equation}
		\tr(\mathbf{S}\mathcal{L}\mathbf{w}) = 	\tr(\mathbf{S}\mathbf{E} \diag(\mathbf{w}) \mathbf{E}^\top) = \tr(\mathbf{R}\diag(\mathbf{w}))
	\end{equation}
	Applying \textbf{Example~19} in~\cite{sun2016majorization} for $\tilde{\mathbf{w}} \ge \bm{0}$, yields
	\begin{multline}
		\left( \mathbf{G} \diag (\tilde{\mathbf{w}}) \mathbf{G}^\top \right)^{-1} \\ 
		\le \mathbf{G}_0^{-1} \mathbf{G} \diag (\tilde{\mathbf{w}}_0) \diag (\tilde{\mathbf{w}})^{-1} \diag (\tilde{\mathbf{w}}_0) \mathbf{G}^\top \mathbf{G}_0^{-1} 
	\end{multline}
	where $\mathbf{G}_0 \triangleq \mathbf{G} \diag(\tilde{\mathbf{w}}_0) \mathbf{G}^\top$.
	Then, applying the cyclic property of the trace, we have
	\begin{equation}
		\tr\left(\mathbf{G}_0 \left( \mathbf{G} \diag (\tilde{\mathbf{w}}) \mathbf{G}^\top \right)^{-1} \right) \le 	\tr\left(\mathbf{Q} \diag(\tilde{\mathbf{w}})^{-1} \right)
	\end{equation}
	with $\mathbf{Q}$ as defined in \eqref{eq:lemma:f1}. Combining these results completes the proof, where the majorizing function for ${f_1(\mathbf{w})}$ is
	\begin{equation}
		\label{eq:f1:major}
		F_1(\mathbf{w}| \tilde{\mathbf{w}}_0) \triangleq  \tr(\mathbf{R}\diag(\mathbf{w})) +  \tr\left(\mathbf{Q} \diag(\tilde{\mathbf{w}})^{-1} \right) 
	\end{equation}
\end{proof}

\begin{lemma}\label{lemma:2}
	The function $f_2(\mathbf{w}) $ in \eqref{prob:whole:init_trans} admits the following upper bound at a given point $\mathbf{w}_0$
	\begin{equation}
		\label{eq:lemma:f2}
		f_2(\mathbf{w}) \le  \mathbf{w}^\top \mathbf{z} + \sigma^2 \sum_{i=1}^{m} \left( \frac{1}{\mathbf{w}_{i0}} \mathbf{w}_{i}^2 + (\log \mathbf{w}_{i0} - 2)\, \mathbf{w}_{i} \right)
	\end{equation}
	with equality for $\mathbf{w} = \mathbf{w}_0$, with $ \mathbf{w}_{i0}$ denoting the $i$-th element.
\end{lemma}
\begin{proof}
	The logarithm function is concave, so for any $x, w_0 > 0$ we have $\log x \;\le\; \log w_0 + \frac{1}{w_0}(x - w_0)$. 
	Multiplying both sides by $x > 0$ gives
	\[
	x \log x 
	\;\le\; 
	x \log w_0 + \frac{x(x - w_0)}{w_0}
	= \frac{1}{w_0} x^2 + (\log w_0 - 1)x,
	\]
	Thus, we have $
		\mathbf{w}_{i} \log(\mathbf{w}_{i}) \le \frac{1}{\mathbf{w}_{i0}} \mathbf{w}_{i}^2 + (\log \mathbf{w}_{i0} - 1)\, \mathbf{w}_{i} $, which can directly give the majorizing function for ${f_2(\mathbf{w})}$
	\begin{multline}
		F_2(\mathbf{w}| \tilde{\mathbf{w}}_0) \triangleq  \mathbf{w}^\top \mathbf{z} \\ + \sigma^2 \sum_{i=1}^{m} \left( \frac{1}{\mathbf{w}_{i0}} \mathbf{w}_{i}^2 + (\log \mathbf{w}_{i0} - 2)\, \mathbf{w}_{i} \right)
	\end{multline} 
\end{proof}

Lastly, the penalty term $f_3(\mathbf{w})$ is majorized thank to the concavity of  the SCAD penalty $h_\lambda$: it allows us to construct a linear majorizer via first-order Taylor approximation
\begin{equation}
\label{eq:maj_f3}
	f_3(\mathbf{w}) \le
	F_3(\mathbf{w}|\tilde{\mathbf{w}}_0)
	\triangleq
	\sum_{i} h'_{\lambda}(\mathbf{w}_{i0})\, \mathbf{w}_{i}
	+ \text{const.}
\end{equation}
with equality for $\mathbf{w} = \mathbf{w}_0$.

Combining Lemma~\ref{lemma:1}, Lemma~\ref{lemma:2} and the term \eqref{eq:maj_f3} leads to the following surrogate objective function
\begin{equation}
	\label{eq:optim}
	F(\mathbf{w}^{(k)}) = \alpha F_1(\mathbf{w}| \tilde{\mathbf{w}}_0)  + (1-\alpha) F_2(\mathbf{w}| \tilde{\mathbf{w}}_0)  + \alpha F_3(\mathbf{w}| \tilde{\mathbf{w}}_0)  \\
\end{equation}
Minimizing $F(\mathbf{w}^{(k)})$ yields the update $\mathbf{w}^{(k+1)}$ by solving decoupled minimization problems over each component of $\mathbf{w}$.
Specifically, the stationary condition for the $i$-th component is given by
\begin{multline}
	0 = \frac{\partial F(\mathbf{w}^{(k)}) }{\partial \mathbf{w}_i} = \alpha \left(\mathbf{R}_i - \frac{\mathbf{Q}_i}{ \mathbf{w}_{i}^2} +  h'_{\lambda}\left(\mathbf{w}_{i0} \right) \right) + \\ (1-\alpha) \left(\mathbf{z}_{i} + \sigma^2 \left( \frac{2}{ \mathbf{w}_{i0}}\mathbf{w}_{i} + \log \mathbf{w}_{i0} - 2\right) \right)
\end{multline}
where \(\mathbf{R}_i = [\diag(\mathbf{R})]_i\) and \(\mathbf{Q}_i = [\diag(\mathbf{Q}_m)]_i\) with $\mathbf{Q}_{m}$ denoting the leading $m \times m$ principal submatrix of $\mathbf{Q}$, i.e., $\mathbf{Q}_{m} = \mathbf{Q}_{1:m, 1:m}$.
Then, defining
\begin{gather}
	a_i \triangleq (1-\alpha)\frac{2\sigma^2}{\mathbf{w}_{i0}}, \label{eq:update_a_i}  \\ 
	C_i \triangleq \alpha\big(\mathbf{R}_i + h'_{\lambda}(\mathbf{w}_{i0})\big) + (1-\alpha)\big(\mathbf{z}_i + \sigma^2(\log \mathbf{w}_{i0} - 2)\big), \label{eq:update_C_i}
\end{gather}
leads to the cubic equation
\begin{equation}
	a_i \mathbf{w}_i^3 + C_i \mathbf{w}_i^2 - \alpha \mathbf{Q}_i = 0.
	\label{prob:cubic}
\end{equation}
Equation \eqref{prob:cubic} is solved by computing the eigenvalues of the companion matrix~\cite{horn2012matrix}, and retaining the positive real root. The complete MM procedure is summarized in Algorithm~\ref{alg:MM}.

\begin{algorithm}[!h]
	\SetAlgoLined
	\SetKw{KwStop}{Stop}
	\SetKwInput{KwInit}{Init}
	\caption{MM algorithm for problem~\eqref{eq:optim}}
	\label{alg:MM}
	
	\KwIn{
		$\mathbf{R}$, $\mathbf{S}$, initial vector $\mathbf{w}^{(0)} = \bm{1}$, $\epsilon$, $\alpha$, $\sigma^2$, $\lambda$
	}
	
	\KwInit{$k \leftarrow 0$}
	
	\While{$k \le \text{maxiter}$}{
		\tcp{Majorization step at $\mathbf{w}^{(k)}$}
		Update $\mathbf{Q}$ using \eqref{eq:update_Q} evaluated at $\mathbf{w}^{(k)}$ \;
		
		\tcp{Minimization step}
		\For{each component $i = 1, \ldots, m$}{
			Compute $a_i$ and $C_i$ using \eqref{eq:update_a_i} and \eqref{eq:update_C_i} \;
			Solve the cubic equation \eqref{prob:cubic} and set $\mathbf{w}_i^{(k+1)}$ to its positive real root \;
		}
		
		Update $\mathbf{w}^{(k+1)} = [\mathbf{w}_1^{(k+1)}, \ldots, \mathbf{w}_m^{(k+1)}]^\top$ \;
		
		\uIf{$\| \mathbf{w}^{(k+1)} - \mathbf{w}^{(k)} \|_F \le \epsilon$}{
			\KwStop
		}
		
		$k \leftarrow k + 1$ \;
	}
\end{algorithm}

\section{Experimental results}

In this section, we conduct experiments on real-world financial data to evaluate the performance of our algorithm.
The dataset consists of $p=30$ stocks selected from the S\&P~500 index over the period ranging from January~2018 to July~2018.
This time span yields $n = 200$ daily observations for each stock. The data are represented by a matrix of log-returns, defined entrywise as
$
	\mathbf{X}_{i,j} = \log P_{i,j} - \log P_{i,j-1},
	\label{eq:log_returns}
$
where $P_{i,j}$ denotes the closing price of the $i$-th stock on day $j$.
The $p=30$ stocks are grouped into $3$ sectors : \emph{Utilities}, \emph{Materials}, and \emph{Health Care},
The ground-truth sector labels are obtained according to the Global Industry Classification Standard (GICS). 

In addition to price-based information, we incorporate stock metadata to enrich the similarity structure.
The metadata associated with the selected stocks are retrieved using the \texttt{yfinance} package\footnote{https://ranaroussi.github.io/yfinance/}.
Specifically, we extract the \texttt{description} field for each stock and employ \texttt{Sentence-BERT} \cite{reimers-2019-sentence-bert} to embed the textual descriptions.
Pairwise distances between the resulting metadata representations are then computed and used as side information, denoted by the matrix $\mathbf{Z}$, in the graph learning process.

We evaluate the performance of the proposed method on a stock clustering task, with the objective of recovering the three underlying market sectors.
Clustering performance is assessed using standard evaluation metrics, including modularity (MOD) \cite{newman2006modularity}. 
Additionally, edge detection performance is evaluated through the F-score (FS), defined as
\begin{equation}
	\mathrm{FS} = \frac{2\,\mathrm{tp}}{2\,\mathrm{tp} + \mathrm{fp} + \mathrm{fn}},
	\label{eq:metrics}
\end{equation}
where $\mathrm{tp}$ represents the number of true positives (correctly identified edges), $\mathrm{fp}$ denotes the number of false positives (incorrectly identified edges), and $\mathrm{fn}$ denotes the number of false negatives (missed true edges). 
The F-score takes values in the interval $[0,1]$, with $\mathrm{FS}=1$ indicating perfect recovery of the underlying graph structure.
In this evaluation, the ground-truth graph is assumed to contain edges only between stocks belonging to the same sector.

We first perform cross-validation over a coarse grid of $20$ values of $\lambda$ ranging from $0.1$ to $10$ using only signal-based information.
The value of $\lambda$ yielding the best clustering performance is then set fixed, and not further refined for the rest of the experiments.
In the following, we rather focus on studying the effect of the fusion parameter $\alpha$.

As illustrated in Fig.~\ref{fig:SP500_fusion_scad_result}, we analyze the graph estimated from stock log-returns (price-based information) and from side-information, as well as their jointly optimization, corresponding to $\alpha \in (0,1)$.
When only price data is used (i.e, $\alpha=1$), the method reduces to a purely signal-driven graph learning approach, which is representative of most existing Laplacian-constrained graph learning algorithms \cite{ying2020nonconvex}.
In this case, the \emph{Utilities} cluster (shown in green) is reasonably well separated, whereas the separation between the \emph{Health Care} (red) and \emph{Materials} (blue) sectors remains ambiguous.

By incorporating the side-information (i.e., $\alpha<1$), a clearer separation between the \emph{Health Care} and \emph{Materials} clusters emerges, demonstrating the benefit of fusing heterogeneous sources of information.

However, when relying solely on metadata (the case $\alpha=0$), the resulting clustering becomes less meaningful, as shown in Fig.~\ref{fig:SP500_fusion_scad_result}(a).
This observation highlights that neither price-based data nor metadata alone is sufficient to fully capture the underlying sector structure.

\begin{figure*}[tbp]
	\centering
	\begin{subfigure}{0.33\textwidth}
		\includegraphics[width=.8\textwidth]{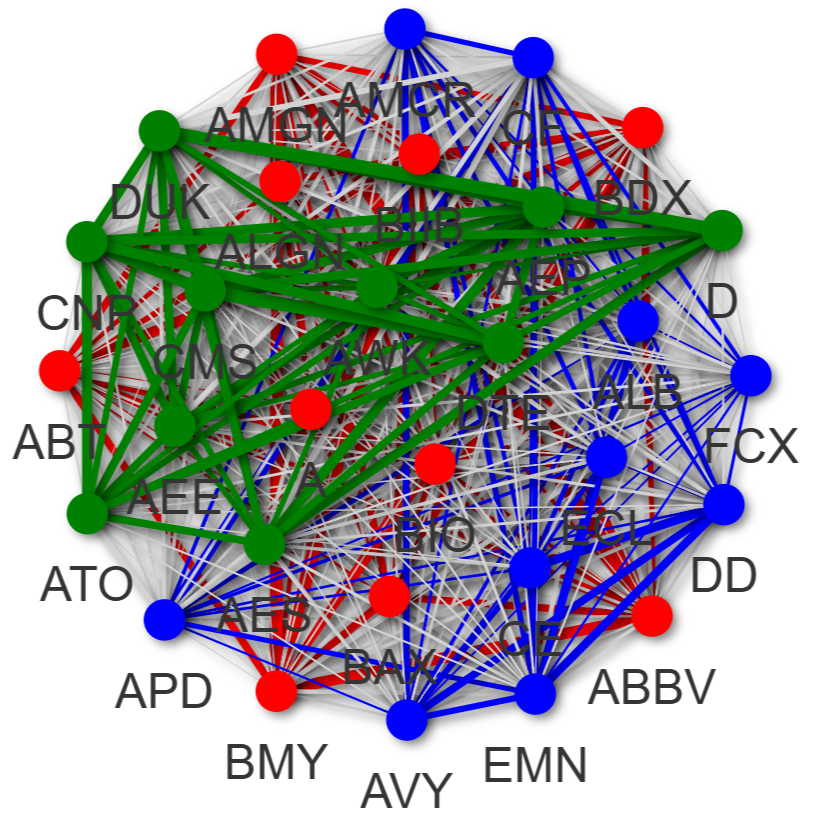}
		\caption{$\alpha=0$}
		\label{fig:SP500_fusion_scad_alpha_0}
	\end{subfigure}%
	\begin{subfigure}{0.33\textwidth}
		\includegraphics[width=\textwidth]{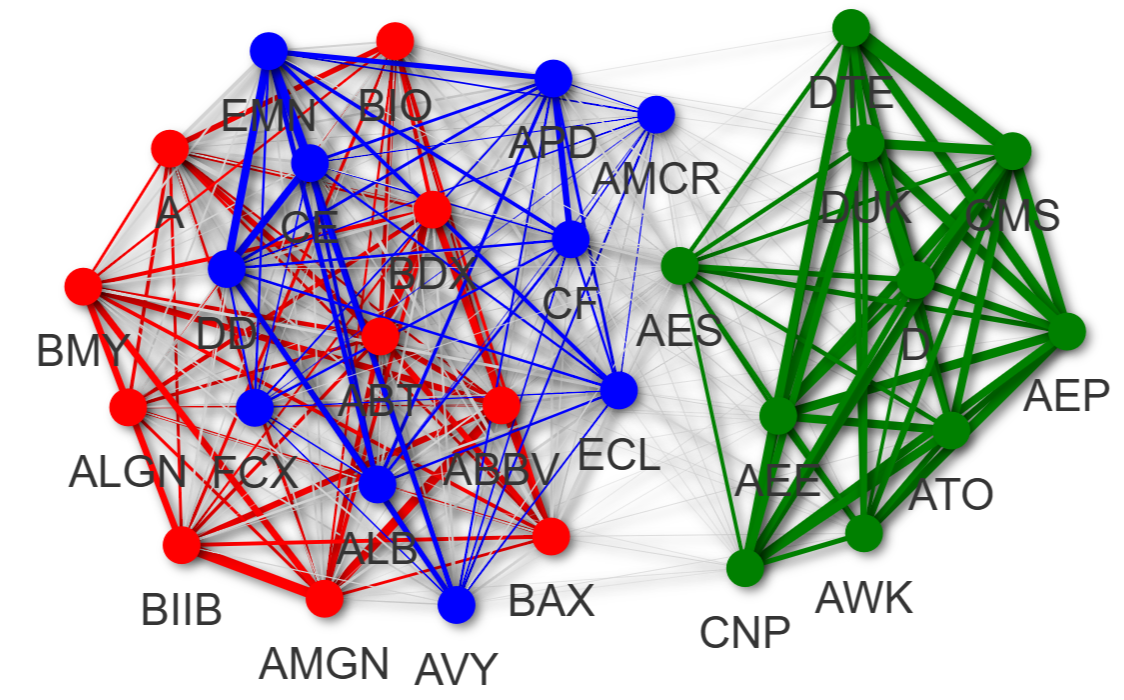}
		\caption{$\alpha=0.4$}
		\label{fig:SP500_fusion_scad_alpha_0.4}
	\end{subfigure}
	\begin{subfigure}{0.33\textwidth}
		\includegraphics[width=.85\textwidth]{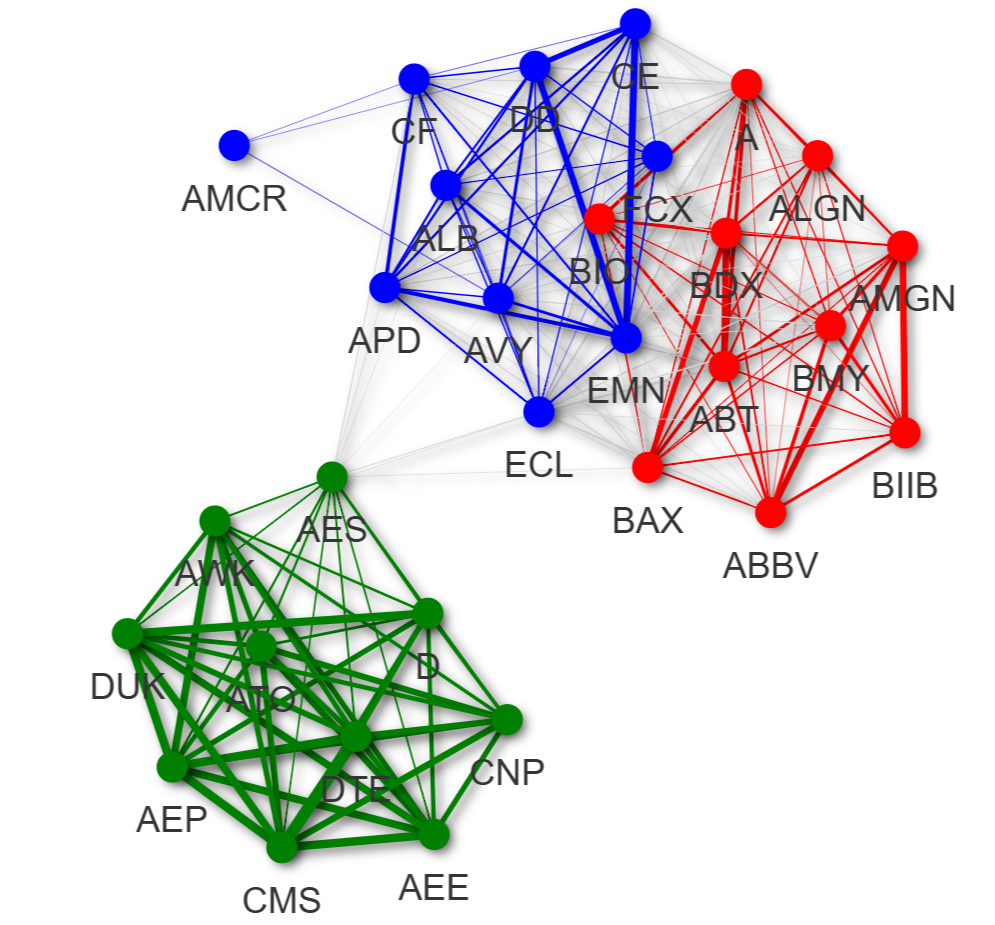}
		\caption{$\alpha=0.6$}
		\label{fig:SP500_fusion_scad_alpha_0.6}
	\end{subfigure}\\
	\begin{subfigure}{0.33\textwidth}
		\includegraphics[width=.85\textwidth]{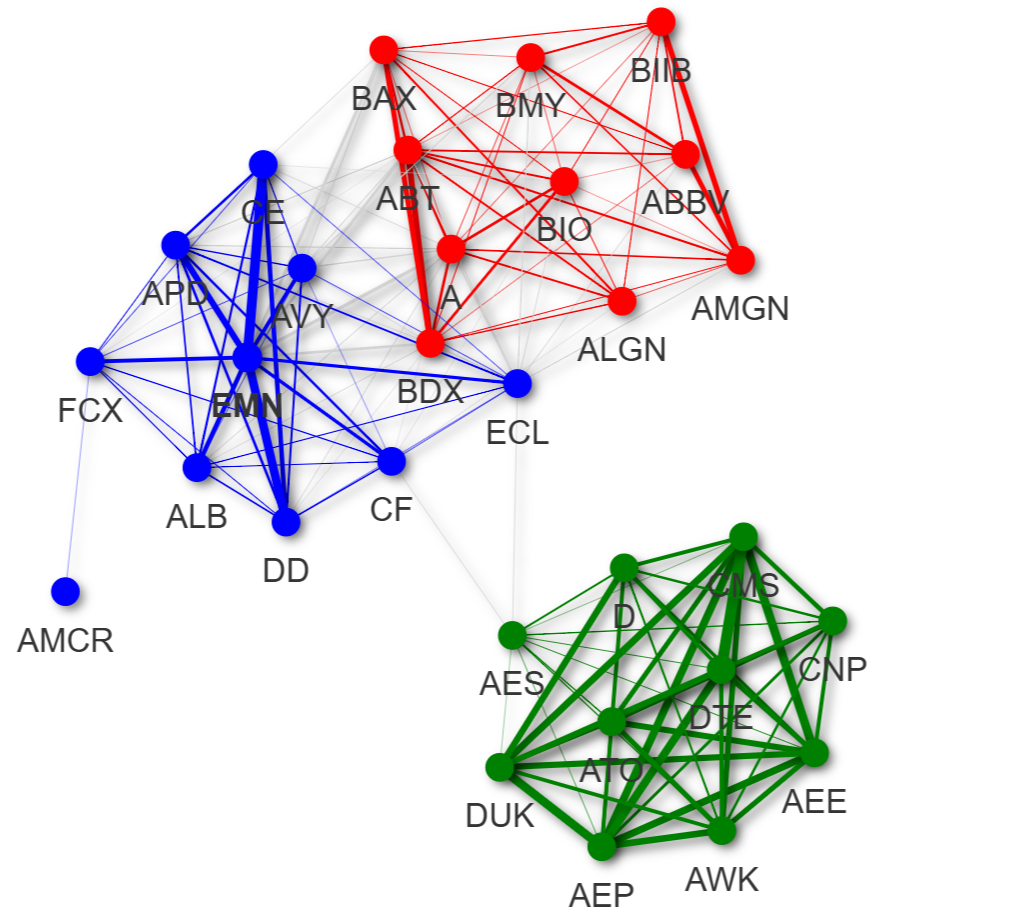}
		\caption{$\alpha=0.8$}
		\label{fig:SP500_fusion_scad_alpha_0.8}
	\end{subfigure}%
	\begin{subfigure}{0.33\textwidth}
		\includegraphics[width=.85\textwidth]{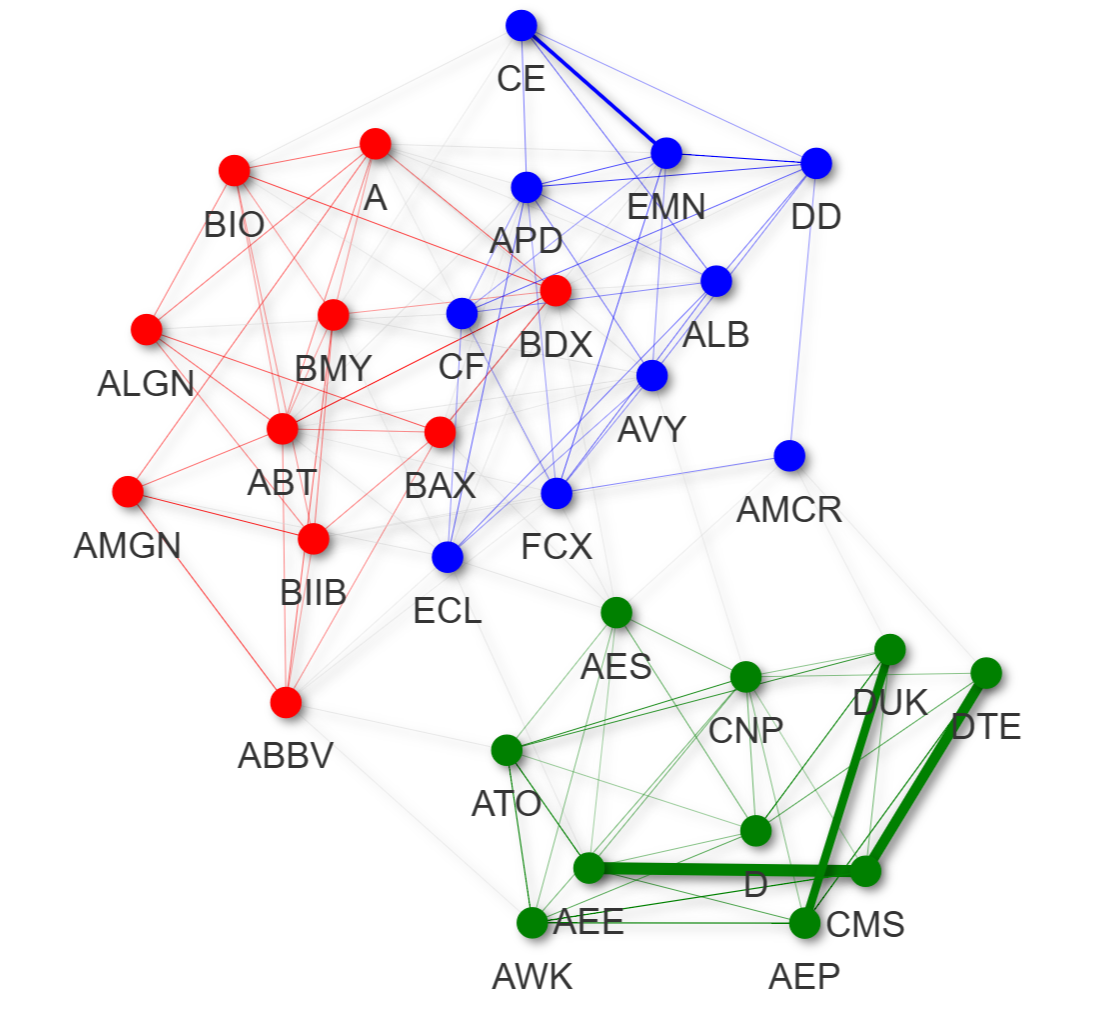}
		\caption{$\alpha=1$}
		\label{fig:SP500_fusion_scad_alpha_1}
	\end{subfigure}%
	\begin{subfigure}{0.33\textwidth}
		\includegraphics[width=\textwidth]{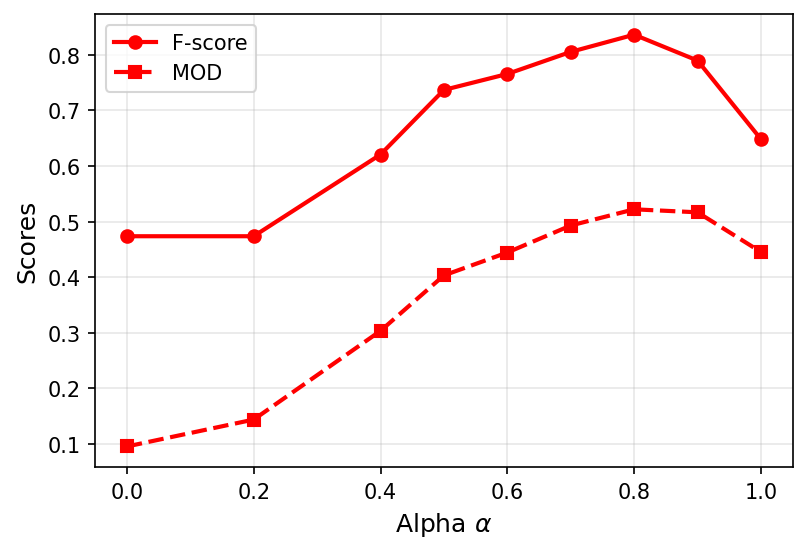}
		\caption{F-score and modularity as function of $\alpha$.}
		\label{fig:SP500_fusion_scad_diff_alpha}
	\end{subfigure}
	\caption{Evolution of the reconstructed graph as a function of $\alpha \in [0,1]$. The extreme case $\alpha = 0$ corresponds to using only side information, whereas $\alpha = 1$ corresponds to using only signal information. Intra-sector connections are represented by the corresponding sector colors, while gray-colored edges indicate connections between nodes belonging to different sectors.}
	\label{fig:SP500_fusion_scad_result}%
\end{figure*}

\section{Conclusion}

In this paper, we proposed a side-information-aware graph learning formulation that simultaneously exploits observed signals and additional node attributes.
To efficiently solve the resulting optimization problem, we developed a MM algorithm, leading to a sequence of tractable subproblems with closed-form updates.
Experimental results on financial data demonstrated that the proposed method significantly improves clustering performance, outperforming existing graph learning approaches that rely only on signal observations.
These results highlight the importance of incorporating heterogeneous side information for enhanced graph inference and downstream learning tasks.
Future work will analyze the joint tuning of the fusion parameter $\alpha$ and the sparsity parameter $\lambda$.

\bibliographystyle{IEEEtran}
\bibliography{references}

\end{document}